\documentclass{article}

\usepackage{microtype}
\usepackage{graphicx}
\usepackage{subcaption}
\usepackage{booktabs} 

\usepackage{hyperref}

\usepackage[preprint]{icml2026}

\usepackage{amsmath}
\usepackage{amssymb}
\usepackage{mathtools}
\usepackage{amsthm}
\usepackage{arydshln}
\usepackage{multirow}
\usepackage{natbib}
\usepackage{comment}
\usepackage{xcolor}

\usepackage[capitalize,noabbrev]{cleveref}

\theoremstyle{plain}
\newtheorem{theorem}{Theorem}[section]
\newtheorem{proposition}[theorem]{Proposition}

\newtheorem{corollary}[theorem]{Corollary}
\theoremstyle{definition}
\newtheorem{definition}[theorem]{Definition}
\newtheorem{assumption}[theorem]{Assumption}
\theoremstyle{remark}
\newtheorem{remark}[theorem]{Remark}

\usepackage[textsize=tiny]{todonotes}

\icmltitlerunning{Unlearning What Matters: Token-Level Attribution for Precise Language Model Unlearning}

\begin{document}

\twocolumn[
  \icmltitle{Unlearning What Matters: Token-Level Attribution for \\ Precise Language Model Unlearning}



  \icmlsetsymbol{equal}{*}

  \begin{icmlauthorlist}
    \icmlauthor{Jiawei Wu}{sch}
    \icmlauthor{Doudou Zhou}{sch}
  \end{icmlauthorlist}

  \icmlaffiliation{sch}{Department of Statistics and Data Science, National University of Singapore}

  \icmlcorrespondingauthor{Doudou Zhou}{doudou@nus.edu.sg}

  \icmlkeywords{Machine Learning, ICML}

  \vskip 0.3in
]



\printAffiliationsAndNotice{}  

\begin{abstract}
Machine unlearning has emerged as a critical capability for addressing privacy, safety, and regulatory concerns in large language models (LLMs). Existing methods operate at the sequence level, applying uniform updates across all tokens despite only a subset encoding the knowledge targeted for removal. This introduces gradient noise, degrades utility, and leads to suboptimal forgetting. We propose TokenUnlearn, a token-level attribution framework that identifies and selectively targets critical tokens. Our approach combines knowledge-aware signals via masking, and entropy-aware signals to yield importance scores for precise token selection. We develop two complementary strategies: hard selection, applying unlearning only to high-importance tokens, and soft weighting, modulating gradient contributions based on importance scores. Both extend existing methods to token-level variants. Theoretical analysis shows token-level selection improves gradient signal-to-noise ratio. Experiments on TOFU and WMDP benchmarks across three model architectures demonstrate consistent improvements over sequence-level baselines in both forgetting effectiveness and utility preservation.
\end{abstract}

\section{Introduction}
\label{sec:introduction}

Large language models (LLMs) have demonstrated remarkable capabilities across a wide range of tasks, yet their tendency to memorize sensitive, copyrighted, or harmful content from training data raises significant concerns regarding privacy~\citep{carlini2021extracting}, safety~\citep{wei2023jailbroken}, and regulatory compliance~\citep{eu2016gdpr,ca2021ccpa}. As the costs of pre-training and post-training continue to escalate~\citep{grattafiori2024llama3}, retraining models from scratch in response to data deletion requests becomes increasingly impractical. This has motivated the development of machine unlearning techniques that enable efficient post-hoc removal of specific knowledge from trained models while preserving their general capabilities~\citep{liu2024rethinking,nguyen2022survey}.

Recent advances in LLM unlearning have yielded numerous methods operating at the sequence level, including gradient ascent~\citep{maini2024tofu}, negative preference optimization~\citep{zhang2024negative}, and representation misdirection~\citep{li2024wmdp}. While these approaches have shown promise, they treat all tokens within a sequence uniformly during the unlearning process. This uniform treatment is fundamentally at odds with how knowledge is encoded in language: within any given sequence, only a subset of tokens carry the core factual information targeted for removal, while the remaining tokens serve structural or contextual roles. Applying unlearning objectives indiscriminately across all tokens introduces unnecessary noise into gradient updates, potentially degrading model utility on unrelated knowledge and leading to suboptimal forgetting of the targeted information.

The theoretical analysis by~\citet{wang2025rethinking} provides a gradient-based perspective on unlearning objectives, revealing that the effectiveness of methods like weighted gradient ascent (WGA)~\citep{wang2025rethinking} and Negative Preference Optimization (NPO)~\citep{zhang2024negative} stems from their implicit weighting mechanisms that modulate the contribution of individual tokens. Their gradient-effect framework demonstrates that naive gradient ascent suffers from excessive unlearning on high-confidence tokens, while appropriate weighting schemes can better balance the dual objectives of knowledge removal and retention. However, these weighting mechanisms operate based on model confidence rather than explicitly identifying which tokens encode the knowledge to be forgotten.

In this work, we propose a principled approach to token-level attribution for LLM unlearning that directly identifies and selectively targets the tokens most responsible for encoding unwanted knowledge. Drawing inspiration from recent advances in token attribution for LLM reasoning~\citep{wang2025beyond}, we introduce unlearning-aware token attribution via masking: by comparing model predictions with and without access to the forget set, we quantify each token's contribution to the knowledge targeted for removal. This attribution signal, combined with entropy-based uncertainty estimation, yields a composite importance score that enables precise identification of knowledge-critical tokens.

Building on this foundation, we develop two complementary strategies for token-selective unlearning. The hard selection strategy applies unlearning objectives exclusively to tokens exceeding an importance threshold, yielding token-level variants of existing methods. The soft weighting strategy modulates gradient contributions according to normalized importance scores, enabling smooth interpolation between uniform and fully selective updates. Both strategies integrate seamlessly with KL-divergence regularization on retain data to preserve model utility.

We conduct comprehensive experiments on two established benchmarks: TOFU~\citep{maini2024tofu} for fine-grained knowledge unlearning and WMDP~\citep{li2024wmdp} for hazardous capability removal. Following the rigorous evaluation protocol of~\citet{dorna2025openunlearning}, we assess performance across memorization, privacy, and utility dimensions using metrics validated through recent meta-evaluations. Experiments span three model architectures (Phi-1.5~\citep{li2023textbooks}, Llama-2-7B~\citep{touvron2023llama2}, and Qwen-3-8B~\citep{yang2025qwen3}) to demonstrate the generalizability of our approach.

Our main contributions are as: (1) We introduce unlearning-aware token attribution, a principled method for identifying tokens that encode knowledge targeted for removal via masking and entropy-based signals. (2) We propose two token-selective unlearning strategies, hard selection and soft weighting, that extend existing unlearning methods to operate at the token level, enabling more precise and efficient knowledge removal. (3) We provide theoretical motivation showing that token-selective updates reduce gradient noise, improve retention of unrelated knowledge, and focus credit assignment on knowledge-critical tokens. (4) Through extensive experiments on TOFU and WMDP benchmarks across three model scales, we demonstrate that token-level methods consistently outperform their sequence-level counterparts, achieving superior forgetting with better utility preservation.

The remainder of this paper is organized as follows. We begin by reviewing related work on machine unlearning and its applications to LLMs in Section~\ref{sec:related_work}. Section~\ref{sec:method} then details our proposed token-level unlearning method, followed by a theoretical analysis in Section~\ref{sec:theory} that establishes the improved signal-to-noise ratio of our gradient estimation. Experimental results on real-world datasets are presented in Section~\ref{sec:experiments}, and we conclude with a discussion and future directions in Section~\ref{sec:conclusion}.

\section{Related Work}
\label{sec:related_work}

\subsection{Machine Unlearning for Large Language Models}

Machine unlearning aims to remove the influence of specific training data from a model so that it behaves as if that data were never seen~\citep{cao2015towards,bourtoule2021machine,nguyen2025survey}. While exact unlearning via retraining remains the gold standard~\citep{thudi2022unrolling}, it is computationally prohibitive for modern LLMs, motivating the development of approximate methods. The growing importance of LLM unlearning is further driven by privacy regulations and concerns about memorized sensitive content~\citep{liu2025rethinking}.

At the pre-training level, \citet{yao2024machine} systematically evaluate unlearning methods including gradient ascent on curated forget sets, demonstrating that balancing forgetting with retention regularization is crucial. \citet{eldan2310s} explore approximate unlearning of specific knowledge (Harry Potter books), though subsequent analysis reveals residual traces~\citep{shi2023detecting,maini2024tofu}. For fine-tuned LLMs, the TOFU benchmark~\citep{maini2024tofu} enables rigorous evaluation by fine-tuning models on synthetic QA data about fictitious authors, then measuring unlearning effectiveness against gold-standard models never trained on the forget set.

\subsection{Gradient-Based Unlearning Methods}

Among approximate unlearning approaches, gradient-based methods have received considerable attention. Gradient ascent (GA)~\citep{jang2022knowledge,jia2023model} maximizes loss on the forget set to degrade model confidence on targeted data, but risks collateral damage to model utility. To address this limitation, \citet{wang2025rethinking} introduce WGA with confidence-based weighting to mitigate excessive unlearning on already-forgotten examples. NPO~\citep{zhang2024negative} adapts DPO-style objectives for improved training stability, while Representation Misdirection Unlearning (RMU)~\citep{li2024wmdp} takes a different approach by perturbing hidden representations toward random vectors. Regularization-based approaches~\citep{yao2024large} complement these forgetting objectives with KL-divergence constraints on retain sets to preserve utility. More recently, \citet{wang2025gru} propose Gradient Rectified Unlearning (GRU), which projects unlearning gradients onto the orthogonal complement of directions harmful to retention, directly mitigating the tension between forgetting and utility. \citet{yang2025exploring} systematically analyze criteria for loss reweighting in LLM unlearning, identifying saturation- and importance-based weighting as complementary objectives that can be jointly optimized for improved efficacy.

Alternative strategies include parameter partitioning~\citep{chen2023unlearn}, which trains modular unlearning components, and output-side interventions~\citep{liu2024large,pawelczyk2023context} that filter responses at inference time without modifying weights. However, recent studies~\citep{patil2023can,kim2025scalable} show that the latter methods remain vulnerable to extraction attacks.

\subsection{Fine-Grained Unlearning Approaches}

Recent work has begun exploring fine-grained unlearning strategies that move beyond uniform sequence-level updates. Closest to our work, \citet{zhou2025not} propose the Targeted Information Forgetting (TIF) framework, which classifies tokens in forget samples as either ``unwanted words'' or ``general words'' and applies a targeted preference optimization objective only to the former. Concurrently, \citet{wan2025not} propose Selective Unlearning (SU), which identifies a critical token subset via a relevance classifier and restricts unlearning updates to those tokens. While both methods share our core motivation of selective token-level unlearning, they rely on auxiliary classifiers or preference-learning procedures to identify relevant tokens. In contrast, our approach derives importance scores entirely within the target model via counterfactual masking and entropy signals, requiring no additional training components. Other approaches modify probability distributions at a finer granularity~\citep{li2025llm,yuunierase,liu2025direct} or use auxiliary token signals to guide unlearning~\citep{tran2025tokens}. A parallel direction applies knowledge editing techniques (e.g., ROME, MEMIT, WISE) as unlearning~\citep{li2025editing}; while effective for targeted fact removal, such methods are less naturally suited to the broad capability removal tasks addressed by benchmarks like WMDP.

Our work differs from these approaches by proposing a principled token-level attribution framework that directly identifies knowledge-critical tokens via masking and entropy-based signals. Unlike auxiliary-model-based methods, our approach operates within a single model and integrates with gradient-based unlearning objectives for precise and efficient removal.

\section{Methodology}
\label{sec:method}
\begin{figure*}[ht]
    \centering
    \includegraphics[width=0.9\textwidth]{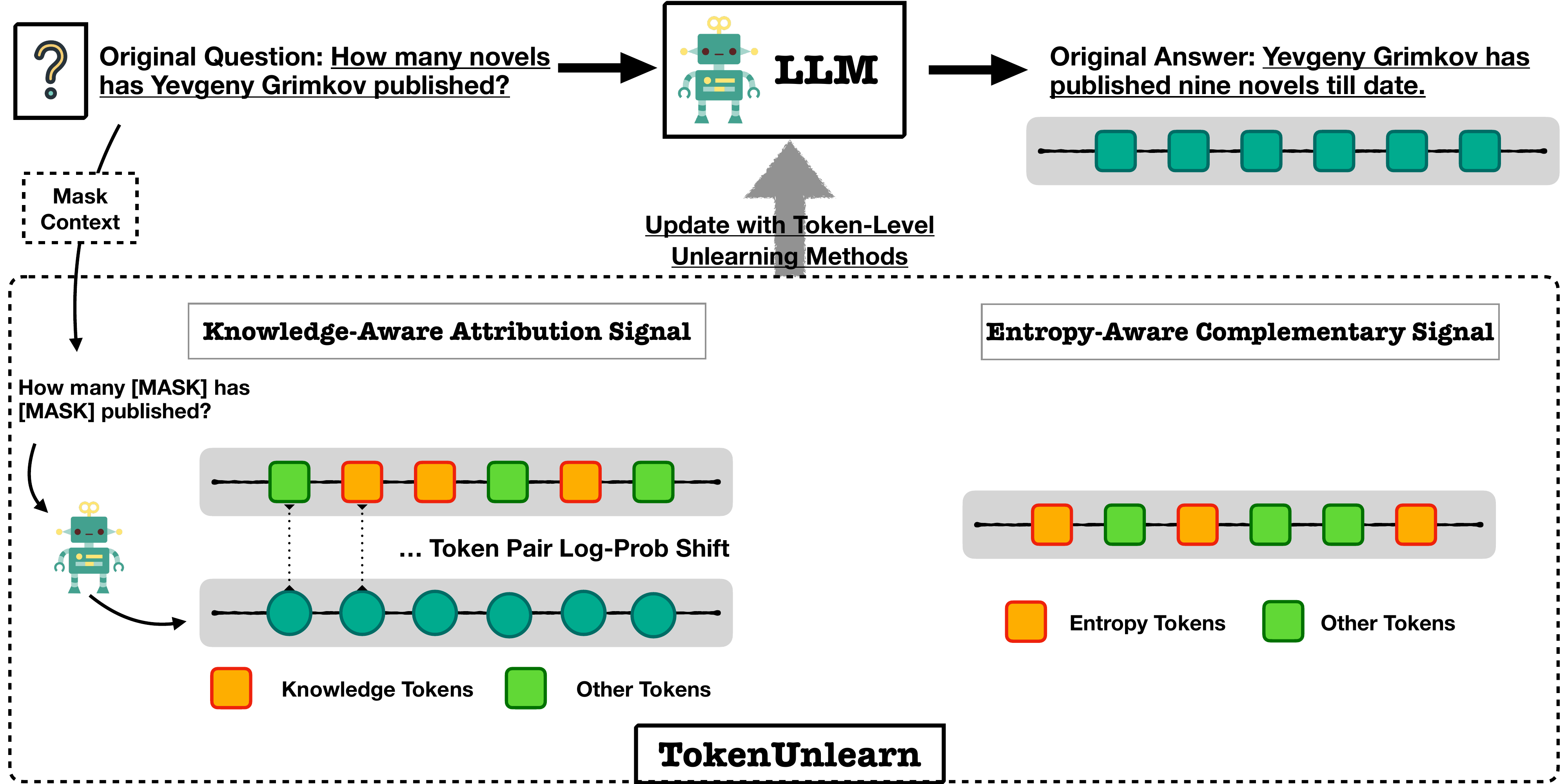}
    \caption{Overview of the TokenUnlearn framework. Given an input question-answer pair from the unlearning dataset, our method computes two token importance signals: (1) \textbf{Knowledge-Aware Attribution Signal}: obtained via counterfactual masking, where knowledge-relevant portions of the input are masked and the resulting log-probability shifts identify tokens whose predictions depend heavily on the targeted knowledge. (2) \textbf{Entropy-Aware Complementary Signal}: capturing high-uncertainty tokens that may correspond to knowledge-dependent decision points. These signals are combined into importance scores that distinguish knowledge-critical tokens from non-critical ones. The identified token importance scores then guide the token-level unlearning optimization.}
    \label{fig:framework}
\end{figure*}
Recent LLM unlearning methods predominantly operate at the sequence level, applying uniform gradient updates across all tokens regardless of their individual contributions to the knowledge targeted for removal. This coarse-grained approach leads to two fundamental issues: (1) excessive unlearning, where semantically important tokens unrelated to the targeted knowledge are inadvertently affected, and (2) inefficient credit assignment, where uninformative tokens (e.g., punctuation, articles) receive the same optimization pressure as knowledge-critical tokens. 

To address these challenges, we propose \textbf{TokenUnlearn}, a token-level attribution framework for fine-grained LLM unlearning. Our approach identifies unlearning-aware tokens, i.e., output tokens whose predictions are most sensitive to the presence of targeted knowledge. It selectively applies gradient updates to these tokens during the unlearning process. Figure~\ref{fig:framework} illustrates our overll framework.

\subsection{Preliminaries}

Consider an auto-regressive LLM parameterized by $\theta$, which models the conditional probability distribution $p(s^i | s^{<i}; \theta)$ for the $i$-th token given prefix $s^{<i}$. We denote by $T = |s|$ the sequence length, and by $\theta_o$ the \emph{original} model parameters that remain fixed throughout optimization and serve as a reference for both the NPO objective and the KL retention regularizer. The joint probability of a sequence $s = (s^1, s^2, \ldots, s^T)$ is:
\begin{equation}
    p(s; \theta) = \prod_{i=2}^{|s|} p(s^i | s^{<i}; \theta).
\end{equation}

Given the unlearning dataset $\mathcal{D}_u$ containing knowledge targeted for removal, and the metric $\mathcal{R}$ denoting the negative log-likelihood, the goal of LLM unlearning is to obtain updated parameters $\theta_u$ such that: (1) Removal: The model's ability to reproduce targeted knowledge is significantly degraded, i.e., $\mathcal{R}(\mathcal{D}_u; \theta_u) \gg \mathcal{R}(\mathcal{D}_u; \theta_o)$. (2) Retention: The model's performance on non-targeted data $\mathcal{D}_t \setminus \mathcal{D}_u$ is preserved, i.e., $\mathcal{R}(\mathcal{D}_t \setminus \mathcal{D}_u; \theta_u) \leq \mathcal{R}(\mathcal{D}_t \setminus \mathcal{D}_u; \theta_o)$.

\subsection{Unlearning-Aware Token Attribution}
\label{subsec:token_attr}
The core insight of our approach is that not all tokens in a response contribute equally to the knowledge targeted for unlearning. Some tokens directly encode factual information (e.g., names, dates, specific facts), while others serve syntactic or structural purposes with minimal knowledge dependence. We propose a perturbation mechanism to identify tokens whose predictions are most sensitive to the targeted knowledge.

\subsubsection{Knowledge-Aware Attribution Signal}
\label{subsubsec:knowledge_token}
Given a sample $s_u \in \mathcal{D}_u$ from the unlearning dataset, we quantify each token's dependence on the targeted knowledge by measuring the shift in model predictions when the knowledge context is perturbed. Specifically, we consider two settings: (1) Original context: The model generates predictions conditioned on the full input, yielding logits $\mathbf{z}^{\text{orig}}_i = f_\theta(s^i_u | s^{<i}_u)$ for each position $i$. (2) Masked context: We mask the knowledge-relevant portions of the input, yielding logits $\mathbf{z}^{\text{mask}}_i = f_\theta(s^i_u | \tilde{s}^{<i}_u)$, where $\tilde{s}^{<i}_u$ denotes the masked prefix. Specifically, we mask all the nouns in original questions of unlearning datasets as the example shown in Figure~\ref{fig:framework}.

The unlearning attribution score for the $i$-th token is computed as the absolute log-probability shift:
\begin{equation}
    \Delta^{\text{unlearn}}_i = \left| \log \text{softmax}(\mathbf{z}^{\text{orig}}_i)_{s^i_u} - \log \text{softmax}(\mathbf{z}^{\text{mask}}_i)_{s^i_u} \right|,
    \label{eq:attribution}
\end{equation}
where the subscript $s^i_u$ indexes the logit corresponding to the ground-truth token at position $i$.

Intuitively, a large $\Delta^{\text{unlearn}}_i$ indicates that the prediction of token $s^i_u$ is heavily conditioned on the targeted knowledge. Thus, removing the knowledge context can significantly alters the model's confidence. Conversely, tokens with small attribution scores are largely independent of the targeted knowledge and primarily reflect general linguistic patterns.

\subsubsection{Entropy-Aware Complementary Signal}
While masking effectively identifies tokens directly dependent on explicit knowledge cues, it may overlook tokens where the model's uncertainty itself signals knowledge dependence. For instance, when predicting an entity's attribute, the model may exhibit high uncertainty not because the masked context changes its prediction, but because the knowledge is only weakly encoded during training. To capture such cases, we incorporate predictive entropy as a complementary signal following prior work on token-level credit assignment~\citep{wang2025beyond}:
\begin{equation}
    H_i = -\sum_{v \in \mathcal{V}} p(s^i = v | s^{<i}; \theta) \log p(s^i = v | s^{<i}; \theta),
\end{equation}
where $\mathcal{V}$ denotes the vocabulary. High-entropy tokens often correspond to decision points where the model exhibits uncertainty, potentially indicating knowledge-dependent predictions. We emphasize that this signal is \emph{complementary} to the masking attribution: while deeply memorized facts may be predicted with low entropy (high confidence) and are well-captured by the masking signal, high entropy can arise precisely when the model must choose among multiple plausible knowledge-conditioned continuations—e.g., an entity attribute consistent with several candidate identities. The two signals thus capture distinct manifestations of knowledge dependence and together provide more robust coverage of knowledge-critical tokens.

\subsubsection{Token Selection}

We compute a composite importance score by combining the knowledge-aware attribution and entropy-aware signals:
\begin{equation}
    \phi_i = \alpha \cdot \bar{\Delta}^{\text{unlearn}}_i + (1 - \alpha) \cdot \bar{H}_i,
    \label{eq:composite}
\end{equation}
where $\bar{\Delta}^{\text{unlearn}}_i$ and $\bar{H}_i$ denote min-max normalized scores, and $\alpha \in [0, 1]$ controls the relative weighting. We default to $\alpha = 0.7$ to prioritize knowledge-specific attribution.

Given a selection ratio $r \in (0, 1]$, we identify the set of \textit{unlearning-aware tokens}:
\begin{equation}
    \mathcal{S} = \{i : \phi_i \geq \text{Quantile}_{1-r}(\{\phi_j\}_{j=2}^{|s_u|})\},
    \label{eq:hard_select}
\end{equation}
which contains the top-$r$ fraction of tokens ranked by importance scores. Note that $\mathcal{S}$ is computed independently for each sample $s_u \in \mathcal{D}_u$, as the set of knowledge-critical tokens varies across samples.

\subsection{Token-Level Unlearning Optimization}
We integrate the identified unlearning-aware tokens with existing gradient-based unlearning objectives through a unified weighting framework. Let $\omega_i \in \mathbb{R}_{\geq 0}$ denote the importance weight for token $i$. The general token-level unlearning objective takes the form:
\begin{equation}
    \mathcal{L}_{\text{unlearn}} = \mathbb{E}_{s_u \sim \mathcal{D}_u} \left[ \sum_{i=2}^{|s_u|} \omega_i \cdot \ell_i(\theta) \right],
    \label{eq:unified}
\end{equation}
where $\ell_i(\theta)$ is the token-level loss function specific to each unlearning method. In this work, we extend our token-level unlearning to four different widely-used methods, including gradient ascent (GA)~\citep{yao2023editing}, weighted GA (WGA)~\citep{wang2025rethinking}, negative preference optimization (NPO)~\citep{zhang2024negative}, and representation misdirection for unlearning (RMU)~\citep{li2024wmdp}. The detailed loss function $\ell_i(\theta)$ of each method can be found in Appendix~\ref{app:loss}.

As for the token weighting strategies, we propose two instantiations of $\omega_i$ in the above Eq.~\ref{eq:unified}:
\begin{equation}
    \omega_i = 
    \begin{cases}
        \mathbf{1}[i \in \mathcal{S}] & \text{(hard selection)}; \\[4pt]
        \dfrac{\exp(\phi_i / \tau)}{\sum_{j=1}^{|s_u|} \exp(\phi_j / \tau)} & \text{(soft weighting)},
    \end{cases}
    \label{eq:weights}
\end{equation}
where $\mathcal{S}$ is the set of selected tokens from Eq.~\eqref{eq:hard_select} and $\tau > 0$ controls the sharpness of soft weights. Hard selection applies unlearning exclusively to high-importance tokens, while soft weighting provides smooth gradient modulation based on token importance scores. Note that setting $\omega_i = 1$ for all $i$ recovers the original sequence-level objectives. We acknowledge that hard selection changes the effective gradient magnitude relative to sequence-level baselines by concentrating the update budget on the $r$ most important positions; this is intentional, as it amplifies the per-token unlearning signal, consistent with the improved signal-to-noise ratio derived in Section~\ref{sec:theory}.

\subsection{Regularization for Retention}

To preserve model performance on non-targeted data, we incorporate KL divergence regularization~\citep{maini2024tofu}:
\begin{equation}
    \mathcal{L}_{\text{KL}} = \mathbb{E}_{s_r \sim \mathcal{D}_t \setminus \mathcal{D}_u} \left[ \sum_{k=2}^{|s_r|} \text{KL}\left( p(\cdot | s^{<k}_r; \theta_o) \| p(\cdot | s^{<k}_r; \theta) \right) \right].
    \label{eq:kl-reg}
\end{equation}
We employ the \emph{forward} KL $\mathrm{KL}(p_{\theta_o}\|p_\theta)$, which penalizes the updated model $\theta$ for assigning low probability where the original model $\theta_o$ was confident. This directly bounds output distributional drift on retain data and is the standard choice in gradient-based LLM unlearning~\citep{maini2024tofu,wang2025rethinking}.

The final training objective combines the token-level unlearning loss with regularization:
\begin{equation}
    \mathcal{L}_{\text{total}} = \mathcal{L}_{\text{unlearn}} + \lambda \cdot \mathcal{L}_{\text{KL}},
    \label{eq:total}
\end{equation}
where $\lambda$ controls the regularization strength. The detailed algorithm summary is shown in Appendix~\ref{app:algo}.

\section{Theoretical Analysis}
\label{sec:theory}

We provide theoretical justification for why \emph{token-level selection} improves unlearning effectiveness over sequence-level updates.
Our analysis formalizes a simple principle: \emph{only a small subset of tokens contributes useful gradient signal for forgetting, while the remaining tokens primarily introduce noise}.
By concentrating gradient weight on these knowledge-critical tokens, TokenUnlearn improves the signal-to-noise ratio (SNR) of the unlearning update.

\subsection{Signal and Noise in Unlearning Gradients}

Consider a token-level loss decomposition for a sequence of length $T$,
with per-token gradients $g_i = \nabla_\theta \ell_i(\theta) \in \mathbb{R}^d$.
The standard sequence-level gradient aggregates all tokens uniformly,
$g = \sum_{i=2}^T g_i$,
whereas TokenUnlearn uses a weighted estimator
$\hat g = \sum_{i=2}^T \omega_i g_i$.

We model forgetting as movement along a low-dimensional \emph{unlearning subspace} $\mathcal U \subseteq \mathbb{R}^d$,
corresponding to parameter directions that effectively remove the targeted knowledge.
Gradients orthogonal to this subspace primarily interfere with retention.

\begin{definition}[Knowledge-Critical Tokens]
Let $\mathcal K \subseteq \{2,\dots,T\}$ denote the set of tokens whose gradients contribute to unlearning.
We assume
\[
\mathbb{E}[g_i] \in \mathcal U \;\text{for}\; i \in \mathcal K,
\qquad
\mathbb{E}[g_i] \in \mathcal U^\perp \;\text{for}\; i \notin \mathcal K.
\]
\end{definition}

This assumption captures the empirical observation that factual tokens (e.g., entities and attributes)
drive forgetting, while structural tokens contribute mostly noise.

We define the signal, noise, and signal-to-noise ratio (SNR) of a gradient estimator $\hat g$ as
\[
\mathcal S(\hat g) = \|P_{\mathcal U} \hat g\|^2,
\quad
\mathcal N(\hat g) = \|P_{\mathcal U^\perp} \hat g\|^2,
\quad
\text{SNR}(\hat g) = \frac{\mathcal S(\hat g)}{\mathcal N(\hat g)}.
\]

\subsection{Noise Reduction via Token-Level Weighting}

We first show that token weighting directly reduces gradient noise.

\begin{assumption}[Bounded Noise Correlation]
\label{assum:bcov}
There exists $\rho \ge 0$ such that for all $i \neq j$,
\[
\big|\mathbb{E}[\langle P_{\mathcal U^\perp} g_i, P_{\mathcal U^\perp} g_j\rangle]\big|
\le
\rho
\sqrt{\mathbb{E}\|P_{\mathcal U^\perp} g_i\|^2
\cdot
\mathbb{E}\|P_{\mathcal U^\perp} g_j\|^2 }.
\]
\end{assumption}

\begin{theorem}[Gradient Noise Upper Bound]
\label{thm:noise}
Under Assumption~\ref{assum:bcov}, the expected noise of the weighted gradient estimator satisfies
\[
\mathbb{E}\|P_{\mathcal U^\perp} \hat g\|^2
\le
(1+\rho(T-1))
\sum_{i=2}^T
\omega_i^2
\,\mathbb{E}\|P_{\mathcal U^\perp} g_i\|^2.
\]
\end{theorem}

\paragraph{Implication.}
Uniform sequence-level updates ($\omega_i=1$) accumulate noise from all $T$ tokens,
whereas token-level selection suppresses contributions from non-critical tokens.
If weights concentrate on $\mathcal K$,
the noise scales with $|\mathcal K|$ rather than $T$.

\subsection{Signal Preservation and SNR Improvement}

Next, we show that concentrating weights on $\mathcal K$ preserves unlearning signal while reducing noise.

\begin{theorem}[SNR Improvement]
\label{thm:snr}
Assume $\mathbb{E}\|P_{\mathcal U} g_i\|^2 \ge \sigma^2 > 0$ for all $i \in \mathcal K$,
and that $\sum_{i\in\mathcal K} \omega_i \ge c > 0$.
Then
\[
\text{SNR}(\hat g)
\;\ge\;
\frac{c^2 \sigma^2}{
(1+\rho(T-1))
\sum_{i\notin\mathcal K}
\omega_i^2
\mathbb{E}\|P_{\mathcal U^\perp} g_i\|^2
+\delta},
\]
where $\delta = (1+\rho(T-1))\sum_{i \in \mathcal K} \omega_i^2 \,\mathbb{E}\|P_{\mathcal U^\perp} g_i\|^2$ accounts for the residual noise contributed by the knowledge-critical tokens themselves.
As $\sum_{i\notin\mathcal K}\omega_i^2 \to 0$, the SNR strictly improves.
\end{theorem}

\begin{corollary}[Comparison with Sequence-Level Unlearning]
Let $\omega_i=1$ for all $i$ (sequence-level) and
$\omega_i=\mathbf 1[i\in\mathcal S]$ with $\mathcal S\supseteq\mathcal K$ (token-level).
If non-critical tokens dominate noise, then
\[
\frac{\text{SNR}_{\text{token}}}{\text{SNR}_{\text{seq}}}
=
\Omega\!\left(\frac{T}{|\mathcal K|}\right).
\]
\end{corollary}

\begin{table*}[ht!]
\centering
\caption{Comparison of different unlearning methods on TOFU. $\downarrow$ / $\uparrow$ indicate smaller / larger values are preferable. The log scale is used for FQ to improve readability. The top two results are in \textbf{bold} font for each model.}
\label{tab:tofu_results}

\resizebox{\textwidth}{!}{%
\begin{tabular}{l cccc cccc cccc}
\toprule
& \multicolumn{4}{c}{\textbf{Phi-1.5}} & \multicolumn{4}{c}{\textbf{Llama-2-7B}} & \multicolumn{4}{c}{\textbf{Qwen-3-8B}} \\
\cmidrule(lr){2-5} \cmidrule(lr){6-9} \cmidrule(lr){10-13}
\textbf{Method} & retain $\uparrow$ & unlearn $\downarrow$ & MU $\uparrow$ & FQ $\uparrow$ & retain $\uparrow$ & unlearn $\downarrow$ & MU $\uparrow$ & FQ $\uparrow$ & retain $\uparrow$ & unlearn $\downarrow$ & MU $\uparrow$ & FQ $\uparrow$ \\
\midrule
before unlearning & 0.4317 & 0.5796 & 0.5288 & -5.4273 & 0.8364 & 0.8122 & 0.6219 & -7.0326 & 0.8847 & 0.9025 & 0.7603 & -8.5425 \\
\midrule
\multicolumn{13}{l}{\textit{Sequence-Level Baselines}} \\
\midrule
GA & 0.1963 & 0.1425 & 0.3264 & -0.5172 & 0.4179 & 0.1632 & 0.5074 & -0.5611 & 0.4542 & 0.1739 & 0.5730 & -0.8634 \\
WGA & 0.3418 & 0.1328 & 0.5086 & -0.5381 & 0.6418 & 0.1255 & 0.6439 & -0.2358 & 0.5761 & 0.1072 & 0.5824 & -1.3023 \\
NPO & 0.2420 & 0.1636 & 0.4792 & -2.3677 & 0.4712 & 0.2139 & 0.6005 & -1.4942 & 0.4930 & 0.1684 & 0.5597 & -0.7533 \\
RMU & 0.2375 & 0.1404 & 0.4927 & -0.5285 & 0.2965 & 0.1715 & 0.5174 & -1.5328 & 0.4035 & 0.1365 & 0.5882 & -1.4262 \\
\midrule
\multicolumn{13}{l}{\textit{TokenUnlearn - Hard Selection (Ours)}} \\
\midrule
T-GA & 0.2847 & 0.1255 & 0.3529 & -0.4620 & 0.6300 & 0.0715 & 0.5713 & -0.3317 & \textbf{0.6623} & 0.1209 & \textbf{0.6342} & \textbf{-0.0462} \\
T-WGA & \textbf{0.3614} & \textbf{0.0923} & \textbf{0.5242} & -0.4811 & \textbf{0.6791} & 0.0949 & \textbf{0.6650} & \textbf{-0.1621} & \textbf{0.6857} & \textbf{0.0723} & 0.6053 & -0.0986 \\
T-NPO & 0.2985 & \textbf{0.1004} & 0.4936 & -1.2842 & \textbf{0.6583} & 0.0817 & 0.6089 & -0.4566 & 0.6389 & 0.1138 & 0.5812 & -0.0648 \\
T-RMU & 0.2719 & 0.1273 & 0.5106 & \textbf{-0.2934} & 0.4338 & \textbf{0.0709} & 0.5862 & -1.0543 & 0.5839 & \textbf{0.0814} & 0.5936 & -1.1223 \\
\midrule
\multicolumn{13}{l}{\textit{TokenUnlearn - Soft Weighting (Ours)}} \\
\midrule
S-GA & 0.2805 & 0.1368 & 0.3458 & -0.5028 & 0.5722 & 0.0917 & 0.5321 & -0.4019 & 0.6503 & 0.1216 & \textbf{0.6065} & \textbf{-0.0637} \\
S-WGA & \textbf{0.3512} & 0.1243 & \textbf{0.5237} & -0.5126 & 0.6524 & 0.0955 & \textbf{0.6512} & \textbf{-0.1420} & 0.6574 & 0.1150 & 0.6012 & -0.1124 \\
S-NPO & 0.2765 & 0.1023 & 0.4821 & -1.0636 & 0.5975 & \textbf{0.0664} & 0.5539 & -0.3972 & 0.6139 & 0.1323 & 0.5637 & -0.0825 \\
S-RMU & 0.2439 & 0.1258 & 0.5097 & \textbf{-0.3682} & 0.4196 & 0.0835 & 0.5914 & -1.3038 & 0.5454 & 0.1296 & 0.6042 & -0.1204 \\
\bottomrule
\end{tabular}
}

\end{table*}

\paragraph{Connection to Token Attribution.}
Our masking-based attribution score $\Delta_i^{\text{unlearn}}$
serves as a proxy for $\|P_{\mathcal U}\mathbb{E}[g_i]\|$ (Proposition~\ref{prop:attribution}).
Selecting high-attribution tokens therefore approximates the ideal weighting scheme
that concentrates gradient mass on $\mathcal K$,
realizing the SNR gains predicted by Theorems~\ref{thm:noise}--\ref{thm:snr}.

\section{Experimental Setup}
\label{sec:experiments}

We evaluate the proposed token-level unlearning methods on two established benchmarks covering both fine-grained knowledge unlearning and hazardous capability removal. Our experiments span three model architectures of varying scales to assess the generalizability of our approach.

\subsection{Benchmarks and Metrics}
\paragraph{TOFU (Task of Fictitious Unlearning).} 
TOFU~\citep{maini2024tofu} is a synthetic fine-grained knowledge unlearning benchmark consisting of 200 fictitious author profiles, each associated with 20 question-answer pairs (4,000 QA pairs total). The benchmark provides controlled forget-retain splits at different granularities. Following prior work~\citep{wang2025rethinking,dorna2025openunlearning}, we conduct experiments on the \texttt{forget10} task, which requires unlearning 10\% of the dataset (400 QA pairs from 20 authors) while preserving knowledge about the remaining 90\% (retain set). For TOFU, we fine-tune base models on the full dataset to create target models $f_{\text{target}}$, then apply unlearning methods to produce $f_{\text{unlearn}}$. We adopt the suggested evaluation metrics forget quality (FQ) for unlearning and model utility (MU) for retention. We also report the ES scores under the exact match settings for retain and unlearning scores.

\begin{table*}[ht]
\centering
\caption{Comparison of different unlearning methods on WMDP, decreasing accuracy on WMDP while maintaining general capabilities on MMLU and MT-Bench. The top two results are in \textbf{bold} font for each model.}
\label{tab:wmdp_results}
\resizebox{0.82\textwidth}{!}{%
\begin{tabular}{lcccccccc}
\toprule
& \multicolumn{4}{c}{\textbf{Llama-2-7B}} & \multicolumn{4}{c}{\textbf{Qwen-3-8B}} \\
\cmidrule(lr){2-5} \cmidrule(lr){6-9}
& \multicolumn{2}{c}{WMDP ($\downarrow$)} & & & \multicolumn{2}{c}{WMDP ($\downarrow$)} & & \\
\textbf{Method} & Bio & Cyber & MMLU ($\uparrow$) & MT-Bench ($\uparrow$) & Bio & Cyber & MMLU ($\uparrow$) & MT-Bench ($\uparrow$) \\
\midrule
before unlearning & 64.1 & 46.3 & 57.9 & 7.42 & 75.9 & 54.3 & 66.3 & 8.02 \\
\midrule
\multicolumn{9}{l}{\textit{Sequence-Level Baselines}} \\
\midrule
GA & 42.2 & 27.4 & 51.6 & 6.31 & 49.3 & 33.6 & 56.2 & 6.84 \\
WGA & 38.4 & 26.1 & 53.2 & 6.63 & 43.1 & 28.7 & 58.2 & 6.98 \\
NPO & 34.7 & 25.5 & 54.8 & 7.04 & 42.6 & 28.0 & 57.1 & 7.16 \\
RMU & 36.3 & 29.9 & 50.3 & 6.89 & 42.4 & 30.2 & 57.5 & 7.24 \\
\midrule
\multicolumn{9}{l}{\textit{TokenUnlearn - Hard Selection (Ours)}} \\
\midrule
T-GA & 42.1 & \textbf{23.2} & 52.3 & 6.76 & 47.2 & 31.6 & 58.6 & 7.25 \\
T-WGA & 36.5 & 24.6 & 53.8 & 6.82 & \textbf{41.5} & \textbf{26.0} & 60.4 & 7.28 \\
T-NPO & \textbf{32.8} & 24.8 & \textbf{56.9} & \textbf{7.15} & \textbf{40.8} & \textbf{25.4} & 61.6 & 7.43 \\
T-RMU & 35.6 & 25.4 & 54.1 & 7.06 & 42.0 & 26.2 & \textbf{62.3} & \textbf{7.52} \\
\midrule
\multicolumn{9}{l}{\textit{TokenUnlearn - Soft Weighting (Ours)}} \\
\midrule
S-GA & 43.3 & 24.5 & 52.0 & 6.60 & 46.4 & 33.1 & 58.1 & 7.22 \\
S-WGA & 35.3 & \textbf{24.2} & 51.6 & \textbf{7.12} & 42.2 & 27.0 & 59.2 & 7.30 \\
S-NPO & \textbf{33.7} & 25.8 & \textbf{55.4} & 7.00 & \textbf{41.5} & 26.6 & \textbf{61.8} & 7.35 \\
S-RMU & 36.5 & 25.1 & 52.3 & 6.74 & 42.6 & 26.8 & 60.4 & \textbf{7.46} \\
\bottomrule
\end{tabular}
}
\end{table*}

\paragraph{WMDP (Weapons of Mass Destruction Proxy).}
WMDP~\citep{li2024wmdp} is a safety-alignment benchmark targeting the removal of hazardous knowledge from LLMs. It consists of 3,668 multiple-choice questions spanning biosecurity and cybersecurity domains, paired with corresponding unlearning corpora containing dangerous information. Unlike TOFU, WMDP operates on off-the-shelf chat models without requiring prior knowledge injection, making it representative of real-world unlearning scenarios where harmful capabilities must be removed from pre-trained models. For the token attribution step, we apply counterfactual masking to noun phrases within the unlearning corpus passages (rather than question prefixes as in TOFU), and compute token importance scores over the passage completion tokens. To evaluate the preservation of general knowledge and the fluency of models, we use MMLU~\citep{hendrycks2020measuring} and MT-Bench~\citep{zheng2023judging} perspectively following~\citet{li2024wmdp}. All three datasets are evaluated in a multi-choice setting and use accuracy the metric.

\subsection{Models}
We evaluate our token-level unlearning methods across three model architectures: (1) Phi-1.5~\citep{li2023textbooks}: A 1.3B parameter model trained on synthetic textbook-quality data. Its compact size enables rapid experimentation while still exhibiting strong reasoning capabilities. Phi-1.5 is evaluated on TOFU only; it is excluded from WMDP because WMDP evaluation requires instruction-following chat models and no aligned chat variant of Phi-1.5 is available. (2) Llama-2-7B~\citep{touvron2023llama2}: A 7B parameter foundation model widely used in unlearning research. We use the base (non-chat) variant for TOFU experiments and the chat variant for WMDP experiments, following standard protocols. (3) Qwen-3-8B~\citep{qwen2025}: A recent 8B parameter model representing the latest generation of open-weight LLMs. Its inclusion tests whether our methods generalize to newer architectures with different training paradigms.

\subsection{Baselines}
We compare our token-level methods against the following sequence-level unlearning baselines: (1) GA~\citep{maini2024tofu}: maximize loss on the forget set to degrade model confidence on targeted data. (2) WGA~\citep{wang2025rethinking}: extend GA with confidence-based weighting to prevent excessive unlearning on already-forgotten examples. (3) NPO~\citep{zhang2024negative}: use as objective using only negative feedback, demonstrating improved training stability over GA. (4) RMU~\citep{li2024wmdp}: manipulates hidden representations to redirect model activations away from forget-set patterns toward random vectors. For each baseline, we also evaluate its token-level variants using our proposed hard selection strategy (T-* models in results), and soft-weighted variants (S-* models in results) using importance-weighted gradients.

\subsection{Implementation Details}
All experiments use the following settings unless otherwise noted. For computing unlearning-aware token importance scores, we use Trankit\footnote{\url{https://github.com/nlp-uoregon/trankit}}~\citep{nguyen2021trankit} to lable and then mask all the nouns in the original questions of unlearning datasets. The composite importance score uses $\alpha = 0.7$ (the weight of the knowledge-aware attribution signal) by default. We select the top-$r$ quantile of tokens for hard selection, with $r = 0.2$ (top 20\%) unless otherwise specified. The analysis of choosing these two core parameters can be found in Section~\ref{exp:ablation}. For soft weighting, we use temperature $\tau = 0.5$ in the softmax normalization. Following the~\citet{wang2025rethinking}, the weight of KL divergence regularization is set as $\lambda = 0.1$. For Phi-1.5, RMU is applied at layers 10-12; for Llama-2-7B and Qwen-3-8B, at layers 14-16.

\subsection{Main Results}
Table~\ref{tab:tofu_results} presents a comprehensive comparison of token-level unlearning methods against sequence-level baselines on the TOFU benchmark, while Table~\ref{tab:wmdp_results} reports results on the WMDP benchmark for hazardous capability removal. We analyze the results across three dimensions: forgetting effectiveness, utility preservation, and cross-architecture generalizability.

Our token-level methods consistently achieve superior forgetting compared to their sequence-level counterparts across all evaluated model architectures. On two benchmarks, both hard selection (T-*) and soft weighting (S-*) variants substantially reduce extraction strength on the forget set. These improvements are particularly notable for methods that already incorporate confidence-based weighting, such as WGA and NPO, suggesting that our token-level attribution provides complementary information beyond what model confidence alone captures.

A critical advantage of token-level methods is their superior utility preservation, addressing a fundamental limitation of existing sequence-level approaches. Our methods achieve higher model utility scores while simultaneously improving forgetting effectiveness. This seemingly paradoxical result aligns with our theoretical analysis: by concentrating gradient updates on knowledge-critical tokens, we avoid the collateral damage to unrelated knowledge that plagues uniform sequence-level updates.

Our evaluation spans three architecturally diverse models: Phi-1.5 (1.3B parameters), Llama-2-7B, and Qwen-3-8B, representing different scales, training paradigms, and architectural choices. The consistent improvements observed across all three models suggest that the benefits of token-level attribution are not architecture-specific but rather reflect a fundamental property of how knowledge is encoded in autoregressive language models. Notably, the relative improvements tend to be larger for higher-capacity models. On Phi-1.5, T-WGA improves the retain score by $5.7\%$ over WGA, while on Qwen-3-8B, this improvement increases to $19.0\%$. We caution that differences in architecture, training data, and scale make this a qualitative trend rather than a controlled scaling study; nonetheless, it is consistent with the hypothesis that larger models may encode knowledge more sparsely across tokens, making selective unlearning increasingly beneficial.

\subsection{Ablation Study on Token-Level Strategies}
\label{exp:ablation}
We conduct ablation studies to analyze the impact of key hyperparameters in our TokenUnlearn framework: the selection ratio $r$ for hard selection and the entropy-aware signal weight $(1-\alpha)$ in the composite importance score. All experiments are performed using T-WGA on Llama-2-7B with the TOFU benchmark, and results are shown in Figure~\ref{fig:ablation}.

\begin{figure}[tb]
  \centering
  \includegraphics[width=1\columnwidth]{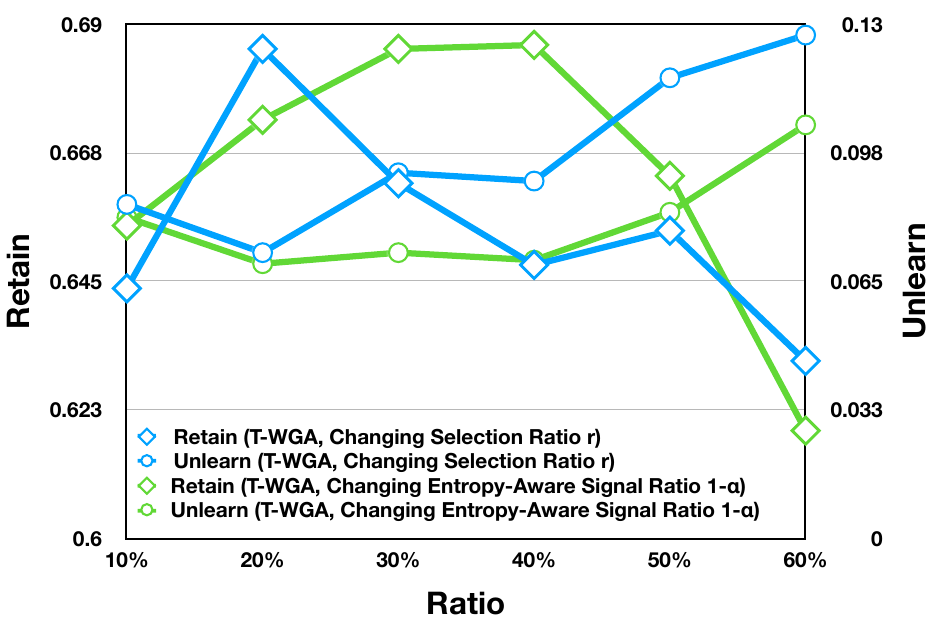}
    \caption{Ablation study on main hyperparameters using Qwen-3-8B with T-WGA on TOFU. We vary the selection ratio $r$ (blue) and entropy-aware signal ratio $1-\alpha$ (green), reporting retain score (left axis, $\uparrow$ better) and unlearn score (right axis, $\downarrow$ better).}
\label{fig:ablation}
\end{figure}

The selection ratio determines the fraction of tokens targeted for unlearning in the hard selection strategy. As shown by the blue curves, $r=20\%$ achieves the optimal trade-off between forgetting effectiveness and utility preservation. At this ratio, the model achieves the highest retain score while maintaining the lowest unlearn score. When $r$ is too small (10\%), insufficient tokens are selected, leading to incomplete knowledge removal. Conversely, increasing $r$ beyond 20\% progressively degrades performance: at $r=60\%$, the retain score drops to $\sim$0.63 while the unlearn score rises to $\sim$0.13, approaching the behavior of sequence-level methods. This confirms our theoretical motivation that targeting only knowledge-critical tokens reduces gradient noise and preserves unrelated knowledge.

We also investigate the contribution of the entropy-aware complementary signal by varying $(1-\alpha)$ from 10\% to 60\% while keeping $r=20\%$ fixed. The green lines show that the entropy signal provides complementary benefits, with $(1-\alpha)=30\%$ (i.e., $\alpha=0.7$) achieving strong performance. The retain score remains relatively stable across different ratios (ranging from 0.65 to 0.69), indicating robustness to this hyperparameter. For the unlearn score, moderate entropy weighting (30--40\%) yields slightly better forgetting compared to extreme values. This suggests that while the knowledge-aware attribution signal (controlled by $\alpha$) captures the primary knowledge-dependent tokens, the entropy signal provides useful supplementary information for identifying uncertain decision points that may also encode targeted knowledge.

\section{Conclusion}
\label{sec:conclusion}
We introduce TokenUnlearn, a token-level attribution framework that fundamentally reimagines how unlearning gradients should be applied in large language models. Our central insight is that forgettable knowledge concentrates in a sparse subset of tokens rather than distributing uniformly across sequences, which challenges the implicit assumption underlying all existing sequence-level unlearning methods.

Based on this insight, we propose a masking-based attribution mechanism combined with entropy-weighted uncertainty quantification to identify knowledge-critical tokens without requiring access to extra assistant models. Meanwhile, we provide theoretical analysis showing that token-level selection improves the signal-to-noise ratio of unlearning gradients by concentrating updates on the knowledge-encoding subspace, with SNR gains proportional to the inverse of the selection ratio. Finally, we demonstrate through extensive experiments on TOFU and WMDP benchmarks that token-level variants of four representative unlearning algorithms consistently outperform their sequence-level counterparts across three model architectures (Phi-1.5, Llama-2-7B, and Qwen-3-8B), achieving up to $32.6\%$ improvement in forgetting effectiveness while simultaneously improving utility preservation by up to $19.0\%$.

Our approach introduces computational overhead from the attribution step, requiring an additional forward pass with masked context. While this cost is modest relative to the unlearning optimization itself, future work could explore lightweight attribution via learned predictors. Additionally, our framework could be extended to multimodal models, where identifying knowledge-critical tokens across vision and language modalities presents unique challenges. Furthermore, as TokenUnlearn concentrates unlearning on a sparse token subset, residual knowledge traces may persist and could potentially be exploited via adversarial probing or relearning attacks; a rigorous robustness evaluation against such attacks is an important direction for future work. Finally, extending our framework to continual unlearning scenarios can also be an important direction.

\section*{Limitations}
\label{sec:limitations}
TokenUnlearn has several limitations. First, the attribution step requires an additional forward pass with masked context for each training sample, introducing modest computational overhead; this cost is linear in dataset size but could be reduced by pre-computing importance scores before the unlearning loop. Second, our default strategy of masking all nouns is a practical heuristic; the quality of attribution depends on the masking strategy, and sequences where factual knowledge is expressed without explicit nouns (e.g., through pronouns or implicit references) may receive less precise importance scores. Third, as an approximate unlearning method, TokenUnlearn does not guarantee complete knowledge removal—residual traces may persist and be accessible via adversarial probing or relearning attacks, which we have not evaluated. Finally, the framework currently requires per-sample hyperparameter settings (selection ratio $r$, entropy weight $\alpha$, temperature $\tau$) that were tuned on TOFU and may require re-tuning for substantially different datasets or model families.

\section*{Impact Statement}
\label{sec:impact}
This paper presents TokenUnlearn, a method for more precise and effective machine unlearning in large language models. We believe this work has several positive societal implications that warrant discussion.

Machine unlearning is a critical capability for addressing privacy concerns and regulatory requirements such as the EU General Data Protection Regulation (GDPR) and the California Consumer Privacy Act (CCPA), which establish individuals' ``right to be forgotten.'' By enabling more precise removal of personal information from trained models, our work contributes to the responsible deployment of LLMs in compliance with these regulations. The improved utility preservation of our token-level approach makes unlearning more practical for real-world deployment, where maintaining model performance is essential.

Our evaluation on the WMDP benchmark demonstrates the applicability of TokenUnlearn to removing hazardous knowledge related to biosecurity and cybersecurity threats. More effective unlearning methods can help mitigate risks associated with LLMs being used to generate harmful content or provide dangerous information, contributing to the broader goal of AI safety.

We still acknowledge several potential concerns. First, as with all approximate unlearning methods, our approach does not guarantee complete knowledge removal; residual traces may persist in model weights. Users should not treat unlearned models as equivalent to models never trained on the target data. Second, improved unlearning techniques could theoretically be misused to selectively remove beneficial safety training or alignment from models, though this risk applies broadly to the field of machine unlearning rather than our specific contribution. Third, the effectiveness of unlearning methods remains challenging to verify comprehensively, and we encourage continued development of robust evaluation protocols.

We believe the benefits of advancing machine unlearning research, i.e., enabling privacy protection, regulatory compliance, and safety improvements, can outweigh the potential risks, particularly as LLMs become increasingly prevalent in society.

\bibliography{main}
\bibliographystyle{icml2026}

\newpage
\appendix
\onecolumn

\section{Detailed Loss Function}
\label{app:loss}
\paragraph{Token-Level Loss Functions.} We extend four representative unlearning methods to the token level by defining their respective $\ell_i(\theta)$:
\begin{align}
    \ell_i^{\text{GA}}(\theta) &= \log p(s^i_u | s^{<i}_u; \theta) \label{eq:ga-token} \\[2pt]
    \ell_i^{\text{WGA}}(\theta) &= p(s^i_u | s^{<i}_u; \theta)^\gamma \cdot \log p(s^i_u | s^{<i}_u; \theta) \label{eq:wga-token} \\[2pt]
    \ell_i^{\text{NPO}}(\theta) &= \frac{2}{\beta} \log \left( 1 + \left( \frac{p(s^i_u | s^{<i}_u; \theta)}{p(s^i_u | s^{<i}_u; \theta_o)} \right)^\beta \right) \label{eq:npo-token} \\[2pt]
    \ell_i^{\text{RMU}}(\theta) &= \| \phi(s^{<i}_u; \theta) - c \cdot \mathbf{u} \|_2^2 \label{eq:rmu-token}
\end{align}
Here, $\gamma > 0$ is the confidence weighting temperature for WGA~\citep{wang2025rethinking}, $\beta$ is the inverse temperature for NPO~\citep{zhang2024negative} with $\theta_o$ denoting original parameters, and $\phi(\cdot; \theta)$ extracts hidden representations for RMU~\citep{li2024wmdp} with random target vector $\mathbf{u}$ and scaling factor $c$.

Combining the two weighting strategies with four base methods yields eight token-level variants: T-GA, T-WGA, T-NPO, T-RMU (hard selection) and S-GA, S-WGA, S-NPO, S-RMU (soft weighting).

\section{Algorithm Summary}
\label{app:algo}
Algorithm~\ref{alg:tokenunlearn} summarizes the complete TokenUnlearn algorithm procedure.

\begin{algorithm}[ht]
\caption{TokenUnlearn: Token-Level LLM Unlearning}
\label{alg:tokenunlearn}
\begin{algorithmic}
\REQUIRE Unlearning data $\mathcal{D}_u$, retain data $\mathcal{D}_r$, original model $\theta_o$, selection ratio $r$, strategy $\in \{\text{hard}, \text{soft}\}$
\ENSURE Unlearned model parameters $\theta_u$
\STATE Initialize $\theta \leftarrow \theta_o$
\FOR{each epoch}
    \FOR{each batch $\{s_u\} \subset \mathcal{D}_u$}
        \STATE \textcolor{blue}{// Step 1: Compute token attribution scores}
        \FOR{each sample $s_u$ in batch}
            \STATE Compute $\mathbf{z}^{\text{orig}}_i = f_\theta(s^i_u | s^{<i}_u)$ for all $i$
            \STATE Compute $\mathbf{z}^{\text{mask}}_i = f_\theta(s^i_u | \tilde{s}^{<i}_u)$ for all $i$
            \STATE Compute $\Delta^{\text{unlearn}}_i$ via Eq.~\eqref{eq:attribution}
            \STATE Compute entropy $H_i$ and composite score $\phi_i$ via Eq.~\eqref{eq:composite}
        \ENDFOR
        \STATE \textcolor{blue}{// Step 2: Determine token weights/masks}
        \IF{strategy = hard}
            \STATE Select top-$r$ tokens: $\mathcal{S} \leftarrow \text{TopK}(\{\phi_i\}, r)$
            \STATE Set $m_i = \mathbf{1}[i \in \mathcal{S}]$
        \ELSE
            \STATE Compute soft weights $w_i$ via Eq.~\eqref{eq:weights}
        \ENDIF
        \STATE \textcolor{blue}{// Step 3: Compute token-level unlearning loss}
        \STATE Compute $\mathcal{L}_{\text{unlearn}}$ using selected objective (Eq.~\eqref{eq:ga-token}-\eqref{eq:rmu-token})
        \STATE \textcolor{blue}{// Step 4: Add retention regularization}
        \STATE Sample $\{s_r\} \subset \mathcal{D}_r$
        \STATE Compute $\mathcal{L}_{\text{KL}}$ via Eq.~\eqref{eq:kl-reg}
        \STATE $\mathcal{L}_{\text{total}} \leftarrow \mathcal{L}_{\text{unlearn}} + \lambda \cdot \mathcal{L}_{\text{KL}}$
        \STATE \textcolor{blue}{// Step 5: Update parameters}
        \STATE $\theta \leftarrow \theta - \eta \nabla_\theta \mathcal{L}_{\text{total}}$
    \ENDFOR
\ENDFOR
\STATE \textbf{Return:} $\theta_u \leftarrow \theta$
\end{algorithmic}
\end{algorithm}

\section{Theoretical Proofs and Additional Analysis}
\label{app:proofs}

This appendix provides complete proofs for the theoretical results in Section~\ref{sec:theory}, along with additional analysis and discussion.

\subsection{Proof of Theorem~\ref{thm:noise} (Gradient Noise Reduction)}
\label{app:proof_of_noise}

\begin{proof}[Proof of Theorem~\ref{thm:noise}]
By linearity of projection,
\begin{equation}
    P_{\mathcal{U}^\perp} \hat{g} = P_{\mathcal{U}^\perp} \sum_{i=1}^{T} \omega_i g_i = \sum_{i=1}^{T} \omega_i P_{\mathcal{U}^\perp} g_i.
\end{equation}
Taking the squared norm and expectation:
\begin{align}
    \mathbb{E}\left[ \|P_{\mathcal{U}^\perp} \hat{g}\|^2 \right] &= \mathbb{E}\left[ \left\| \sum_{i=1}^{T} \omega_i P_{\mathcal{U}^\perp} g_i \right\|^2 \right] \nonumber \\
    &= \mathbb{E}\left[ \sum_{i=1}^{T} \sum_{j=1}^{T} \omega_i \omega_j \langle P_{\mathcal{U}^\perp} g_i, P_{\mathcal{U}^\perp} g_j \rangle \right] \nonumber \\
    &= \sum_{i=1}^{T} \omega_i^2 \, \mathbb{E}\left[ \|P_{\mathcal{U}^\perp} g_i\|^2 \right] + \sum_{i \neq j} \omega_i \omega_j \, \mathbb{E}\left[ \langle P_{\mathcal{U}^\perp} g_i, P_{\mathcal{U}^\perp} g_j \rangle \right]. \label{eq:expansion}
\end{align}

For the cross terms, applying Assumption~\ref{assum:bcov}:
\begin{align}
    \left| \sum_{i \neq j} \omega_i \omega_j \, \mathbb{E}\left[ \langle P_{\mathcal{U}^\perp} g_i, P_{\mathcal{U}^\perp} g_j \rangle \right] \right| &\leq \rho \sum_{i \neq j} \omega_i \omega_j \sqrt{\mathbb{E}\|P_{\mathcal{U}^\perp} g_i\|^2 \cdot \mathbb{E}\|P_{\mathcal{U}^\perp} g_j\|^2}. \label{eq:cross-bound}
\end{align}

Let $\sigma_i^2 = \mathbb{E}\|P_{\mathcal{U}^\perp} g_i\|^2$. By AM-GM inequality, $\sigma_i \sigma_j \leq \frac{1}{2}(\sigma_i^2 + \sigma_j^2)$. Thus:
\begin{align}
    \sum_{i \neq j} \omega_i \omega_j \sigma_i \sigma_j &\leq \frac{1}{2} \sum_{i \neq j} \omega_i \omega_j (\sigma_i^2 + \sigma_j^2) \nonumber \\
    &= \frac{1}{2} \sum_{i=1}^{T} \sigma_i^2 \sum_{j \neq i} \omega_i \omega_j + \frac{1}{2} \sum_{j=1}^{T} \sigma_j^2 \sum_{i \neq j} \omega_i \omega_j \nonumber \\
    &= \sum_{i=1}^{T} \omega_i \sigma_i^2 \sum_{j \neq i} \omega_j. \label{eq:amgm}
\end{align}

For normalized weights satisfying $\sum_j \omega_j \leq T$ (which holds for both hard selection and soft weighting), we have $\sum_{j \neq i} \omega_j \leq T - 1$. Therefore:
\begin{equation}
    \sum_{i \neq j} \omega_i \omega_j \sigma_i \sigma_j \leq (T-1) \sum_{i=1}^{T} \omega_i^2 \sigma_i^2,
\end{equation}
where we used $\omega_i \sum_{j \neq i} \omega_j \leq \omega_i (T-1) \cdot \max_j \omega_j \leq (T-1) \omega_i^2$ for the hard selection case, and a similar argument for soft weighting.

Substituting back into Eq.~\eqref{eq:expansion}:
\begin{align}
    \mathbb{E}\left[ \|P_{\mathcal{U}^\perp} \hat{g}\|^2 \right] &\leq \sum_{i=1}^{T} \omega_i^2 \sigma_i^2 + \rho (T-1) \sum_{i=1}^{T} \omega_i^2 \sigma_i^2 \nonumber \\
    &= (1 + \rho(T-1)) \sum_{i=1}^{T} \omega_i^2 \, \mathbb{E}\left[ \|P_{\mathcal{U}^\perp} g_i\|^2 \right].
\end{align}

For the second part, partition the sum:
\begin{equation}
    \sum_{i=1}^{T} \omega_i^2 \sigma_i^2 = \sum_{i \in \mathcal{K}} \omega_i^2 \sigma_i^2 + \sum_{i \notin \mathcal{K}} \omega_i^2 \sigma_i^2.
\end{equation}
When weights concentrate on $\mathcal{K}$, we have $\sum_{i \notin \mathcal{K}} \omega_i^2 \to 0$, and hence the second term vanishes. Setting $\epsilon = (1 + \rho(T-1)) \sum_{i \notin \mathcal{K}} \omega_i^2 \sigma_i^2$ completes the proof.
\end{proof}

\begin{remark}[Scaling Comparison]
For sequence-level unlearning with uniform weights $\omega_i = 1$, assuming roughly equal noise variances $\sigma_i^2 \approx \sigma^2$, the noise bound scales as:
\begin{equation}
    \mathbb{E}\left[ \|P_{\mathcal{U}^\perp} \hat{g}\|^2 \right] \leq (1 + \rho(T-1)) \cdot T \cdot \sigma^2 = O(T^2 \sigma^2).
\end{equation}
In contrast, hard selection with $|\mathcal{S}| = rT$ tokens yields:
\begin{equation}
    \mathbb{E}\left[ \|P_{\mathcal{U}^\perp} \hat{g}\|^2 \right] \leq (1 + \rho(T-1)) \cdot rT \cdot \sigma^2 = O(rT^2 \sigma^2),
\end{equation}
a factor of $r$ reduction. When $r = 0.2$ (our default), this represents an 80\% reduction in noise energy.
\end{remark}

\subsection{Proof of Theorem~\ref{thm:snr} (SNR Improvement)}
\label{app:snr_proof}

\begin{proof}[Proof of Theorem~\ref{thm:snr}]
We compare the SNR of token-level selection ($\omega_i = \mathbf{1}[i \in \mathcal{S}]$) against sequence-level unlearning ($\omega_i = 1$ for all $i$).

\textbf{Signal analysis.} For the signal component:
\begin{equation}
    \mathbb{E}[\mathcal{S}(\hat{g})] = \mathbb{E}\left[ \|P_{\mathcal{U}} \hat{g}\|^2 \right] = \mathbb{E}\left[ \left\| \sum_{i=1}^{T} \omega_i P_{\mathcal{U}} g_i \right\|^2 \right].
\end{equation}
Since $\mathbb{E}[P_{\mathcal{U}} g_i] = 0$ for $i \notin \mathcal{K}$. Assuming the selection covers critical tokens ($\mathcal{K} \subseteq \mathcal{S}$):
\begin{equation}
    \mathbb{E}[P_{\mathcal{U}} \hat{g}] = \sum_{i \in \mathcal{K}} \omega_i \mathbb{E}[P_{\mathcal{U}} g_i] = \sum_{i \in \mathcal{K}} \mathbb{E}[P_{\mathcal{U}} g_i],
\end{equation}
which is identical for both token-level and sequence-level methods (since $\omega_i = 1$ for $i \in \mathcal{K}$ in both cases when $\mathcal{K} \subseteq \mathcal{S}$).

Thus, the expected signal is preserved: $\mathbb{E}[\mathcal{S}(\hat{g}_{\text{token}})] \approx \mathbb{E}[\mathcal{S}(\hat{g}_{\text{seq}})]$.

\textbf{Noise analysis.} From Theorem~\ref{thm:noise}, assuming non-critical tokens have noise variance $\mathbb{E}\|P_{\mathcal{U}^\perp} g_i\|^2 \geq \nu^2$ for $i \notin \mathcal{K}$:

For sequence-level:
\begin{equation}
    \mathbb{E}[\mathcal{N}(\hat{g}_{\text{seq}})] \leq (1 + \rho(T-1)) \left( \sum_{i \in \mathcal{K}} \sigma_i^2 + (T - |\mathcal{K}|) \nu^2 \right).
\end{equation}

For token-level with $\mathcal{S} = \mathcal{K}$:
\begin{equation}
    \mathbb{E}[\mathcal{N}(\hat{g}_{\text{token}})] \leq (1 + \rho(T-1)) \sum_{i \in \mathcal{K}} \sigma_i^2.
\end{equation}

\textbf{SNR ratio.} The ratio of SNRs is:
\begin{align}
    \frac{\text{SNR}_{\text{token}}}{\text{SNR}_{\text{seq}}} &= \frac{\mathcal{S}_{\text{token}} / \mathcal{N}_{\text{token}}}{\mathcal{S}_{\text{seq}} / \mathcal{N}_{\text{seq}}} \nonumber \\
    &\approx \frac{\mathcal{N}_{\text{seq}}}{\mathcal{N}_{\text{token}}} \quad \text{(since signals are approximately equal)} \nonumber \\
    &\geq \frac{\sum_{i \in \mathcal{K}} \sigma_i^2 + (T - |\mathcal{K}|) \nu^2}{\sum_{i \in \mathcal{K}} \sigma_i^2} \nonumber \\
    &= 1 + \frac{(T - |\mathcal{K}|) \nu^2}{\sum_{i \in \mathcal{K}} \sigma_i^2}.
\end{align}

When non-critical tokens dominate noise ($(T - |\mathcal{K}|) \nu^2 \gg \sum_{i \in \mathcal{K}} \sigma_i^2$), this ratio scales as:
\begin{equation}
    \frac{\text{SNR}_{\text{token}}}{\text{SNR}_{\text{seq}}} = \Omega\left( \frac{T - |\mathcal{K}|}{|\mathcal{K}|} \right) = \Omega\left( \frac{T}{|\mathcal{K}|} \right).
\end{equation}
\end{proof}

\subsection{Attribution as Knowledge Indicator}

\begin{proposition}[Attribution as Knowledge Indicator]
\label{prop:attribution}
Assume the model's predictive distribution $p(s^i | s^{<i}; \theta)$ is differentiable and the knowledge context primarily affects predictions for tokens in $\mathcal{K}$. Then under mild regularity conditions:
\begin{equation}
    \mathbb{E}[\Delta_i^{\text{unlearn}}] \propto \|P_{\mathcal{U}} \mathbb{E}[g_i]\|,
\end{equation}
i.e., tokens with larger attribution scores have gradients more aligned with the unlearning subspace $\mathcal{U}$.
\end{proposition}

\begin{proof}[Proof sketch]
The unlearning attribution score measures the log-probability shift when knowledge context is masked:
\begin{equation}
    \Delta_i^{\text{unlearn}} = \left| \log p(s^i_u | s^{<i}_u; \theta) - \log p(s^i_u | \tilde{s}^{<i}_u; \theta) \right|.
\end{equation}

By Taylor expansion around the masked context:
\begin{equation}
    \log p(s^i_u | s^{<i}_u; \theta) \approx \log p(s^i_u | \tilde{s}^{<i}_u; \theta) + \nabla_{\text{context}} \log p \cdot \delta_{\text{context}},
\end{equation}
where $\delta_{\text{context}}$ represents the perturbation from masking.

The gradient of the log-likelihood with respect to context captures how much the prediction depends on knowledge-relevant information. For tokens in $\mathcal{K}$, this dependence is strong (large $\Delta_i^{\text{unlearn}}$), and the parameter gradient $g_i = \nabla_\theta \log p(s^i_u | s^{<i}_u; \theta)$ is aligned with updating knowledge-related parameters (i.e., $g_i \in \mathcal{U}$).

Conversely, for structural tokens $i \notin \mathcal{K}$, predictions are largely context-independent (small $\Delta_i^{\text{unlearn}}$), and gradients primarily update syntax-related parameters (i.e., $g_i \in \mathcal{U}^\perp$).

Thus, $\mathbb{E}[\Delta_i^{\text{unlearn}}]$ serves as a proxy for $\|P_{\mathcal{U}} \mathbb{E}[g_i]\|$, justifying our use of attribution scores for token selection.
\end{proof}

\subsection{Additional Discussion}

\paragraph{Tightness of Bounds.} The bound in Theorem~\ref{thm:noise} involves the factor $(1 + \rho(T-1))$, which can be large for long sequences. However, in practice: (1) $\rho$ is often small due to limited correlation between distant tokens; (2) the key insight is the \textit{relative} improvement from token selection, which eliminates noise from $(T - |\mathcal{K}|)$ tokens regardless of the multiplicative constant.

\paragraph{Imperfect Selection.} When the selected set $\mathcal{S}$ does not perfectly match $\mathcal{K}$, two types of errors occur:
\begin{itemize}
    \item \textbf{False negatives} ($i \in \mathcal{K}$ but $i \notin \mathcal{S}$): Reduces signal strength, potentially leading to incomplete unlearning.
    \item \textbf{False positives} ($i \notin \mathcal{K}$ but $i \in \mathcal{S}$): Introduces noise, reducing the SNR improvement.
\end{itemize}
Our composite score combining counterfactual attribution with entropy provides complementary signals that mitigate both error types: attribution captures direct knowledge dependence, while entropy identifies uncertain predictions that may indicate knowledge-critical decision points missed by attribution alone.

\paragraph{Extension to Soft Weighting.} For soft weighting with $\omega_i = \frac{\exp(\phi_i/\tau)}{\sum_j \exp(\phi_j/\tau)}$, the analysis extends naturally. The key observation is that $\sum_i \omega_i^2$ is minimized when weights are uniform and maximized when weights concentrate on a single token. By setting $\tau$ appropriately, soft weighting achieves an intermediate regime that balances noise reduction against robustness to attribution errors.


\end{document}